\pdfoutput=1
\documentclass[pdflatex,sn-apa]{sn-jnl}

\usepackage{amsmath,amssymb}
\usepackage{graphicx}
\usepackage[english]{babel}
\usepackage{etoolbox}
\apptocmd{\thebibliography}{\let\apaciteNewblock\newblock\def\newblock{\ifhmode\apaciteNewblock\fi}\setlength{\itemsep}{3pt}\small}{}{}

\theoremstyle{thmstyleone}
\newtheorem{theorem}{Theorem}
\newtheorem{proposition}[theorem]{Proposition}
\newtheorem{corollary}[theorem]{Corollary}
\theoremstyle{thmstyletwo}
\newtheorem{remark}{Remark}

\begin{document}

\title[Equivalence testing and run-count design]{Equivalence testing and run-count design for stochastic machine-learning experiments}

\author*[1]{\fnm{Dmytro} \sur{Knopov}}\email{d.knopov@ukma.edu.ua}

\author[1,2]{\fnm{Nadiya} \sur{Shvai}}\email{n.shvay@ukma.edu.ua}

\affil*[1]{\orgdiv{Faculty of Computer Sciences, Department of Mathematics}, \orgname{National University of Kyiv-Mohyla Academy}, \orgaddress{\city{Kyiv}, \country{Ukraine}}}

\affil[2]{\orgname{Cyclope.ai}, \orgaddress{\city{Paris}, \country{France}}}

\abstract{
Claims that two training methods perform equally well often rest on a non-significant difference between two or three runs per method.
We develop a protocol, with its power theory, for certifying practical equivalence in per-run metrics such as the AUROC (area under the receiver operating characteristic curve).
It rests on a two-level (seed~$\times$~item) variance model separating between-run dispersion from item-level noise.
Equivalence within a pre-specified margin is tested by two one-sided tests (TOST), equivalently by a confidence interval inside the margin; a two-level bootstrap covers composite statistics.
A closed-form formula converts a margin and a target power into runs per method; its shifted-$t$ refinement is calibrated against the realised two-level process.
Simulations on a 270-cell grid and under non-Gaussian seed laws confirm the formula within its design-relevant range, from five runs per method.
For four published runs of a joint energy-based model, the failure to certify equivalence was largely decided by design: with the variance calibrated to those runs, two per method certify with probability near three per~cent, whereas certifying 0.01 AUROC at 80~per~cent power needs twelve even under exact equality.
A second case study on a thousand-seed CIFAR-10 zoo and the MultiBERTs checks the formula against real seed laws (within 0.05 in that range) and resolves comparisons that two or three runs leave undecided.
In a two-stage design (pilot, then formula), our experiments suggest that an upper confidence bound on the pilot variance substantially reduces the undersizing seen with its point estimate.
}


\maketitle

\section{Introduction}
\label{sec:intro}

A recurring conclusion of empirical machine-learning research is that a new training method performs ``on par'' with an established baseline.
Classical significance testing cannot support this conclusion: failure to reject the hypothesis of no difference is not evidence of no difference, least of all when the comparison rests on two or three training runs \citep{Colas2018}.
At the same time, the variance of benchmark results, from data sampling and from arbitrary training choices such as initialisation and data order, is a first-order effect in modern benchmarks, often of the same order as the improvements being reported \citep{Bouthillier2021}, and the part induced by the training seed alone has been measured across domains \citep{Dodge2020,Picard2021}.
Small-$n$ comparisons that end in a null outcome are therefore routine, and routinely over-interpreted.

Equivalence testing inverts the burden of proof.
One fixes, before the experiment, a margin $\delta$ below which a difference is declared practically irrelevant; the null hypothesis states that the methods differ by at least $\delta$, and rejecting it \emph{certifies} equivalence at a controlled error level \citep{Schuirmann1987,Lakens2017}.
The question that then decides the fate of the experiment is one of design: how many training runs does the certification require?

The concrete motivation for this work is our experimental study \citep{Knopov2026jem} of joint energy-based models \citep{Grathwohl2020}, which compared two training methods and found their macro-averaged areas under the receiver operating characteristic (ROC) curve (AUROC) within $0.006$ of each other, yet two one-sided tests (TOST) at margin $\delta = 0.01$ returned $p = 0.33$ from $n = 2$ runs per method.
As Section~\ref{sec:case} shows, a variance calibrated retrospectively to its own runs leaves that negative outcome very little room to be otherwise: at $n = 2$ the rejection region of the deterministic power approximation is empty, and the realised probability of certification, in a simulation calibrated to the observed variance, is roughly three per~cent.
Turning this experience into theory yields a protocol that answers, before the certification experiment is run and given a pilot or an external estimate of the run-level dispersion, whether an equivalence claim is within reach.

The contributions of this work are methodological rather than the introduction of a new equivalence test.
We develop a framework for applying equivalence testing to stochastic machine-learning experiments, treating independently trained runs as the experimental units and distinguishing explicitly between the setting with an independent test draw for every run and the shared-benchmark setting used in practice.
First, we formalise these two settings with a two-level (seed~$\times$~item) variance model for per-run AUROC: an exact variance decomposition under independent test draws (Proposition~\ref{prop:variance}) and its conditional counterpart for a shared test set (Proposition~\ref{prop:shared}).
The model separates the seed-level dispersion, which only additional runs can reduce, from the item-level noise, which a larger test set shrinks.
Second, we define a seed-level Welch-TOST protocol with confidence-interval duality (Proposition~\ref{prop:duality}) and size control inherited from the intersection--union principle up to the Welch approximation.
Third, we specify a two-level bootstrap for composite per-run statistics and map its finite-sample validity by simulation, over run counts and variance regimes, rather than asserting it (Remark~\ref{rem:bootvalidity}).
Fourth, we derive an approximate power function (Proposition~\ref{prop:power}) with the design formula~\eqref{eq:n} for the required number of runs, validate it on a $270$-cell simulation grid and benchmark it against the exact pooled-model power of \citet{Shieh2016}.
Fifth, we revisit the data of \citet{Knopov2026jem} in a case study that reproduces the published numbers exactly.
It answers the design question: certifying $\pm 0.01$ at $80\%$ power requires about $12$ runs per method even for truly identical methods, not $2$.
Sixth, in a second case study on two public collections of independently trained models with up to a thousand runs (Section~\ref{sec:case2}), we check the design formula against real, non-Gaussian seed laws, exhibit the shared-test-set regime in which the binomial item-level variance exceeds the observed per-run variance, and show the protocol resolving comparisons that two or three runs leave undecided.
The same case study takes the two-stage pilot design end to end and compares two ways of entering the pilot variance into the formula.
Sections~\ref{sec:model}--\ref{sec:power} develop the theory, Section~\ref{sec:sim} reports the simulations, and Sections~\ref{sec:case} and~\ref{sec:case2} the case studies.

\section{Related work}
\label{sec:related}

Equivalence testing originated in pharmacokinetics, where approving a generic drug requires demonstrating that its bioavailability lies within a pre-specified window of the reference formulation.
Schuirmann's two one-sided tests \citep{Schuirmann1987} remain the standard instrument; their size properties follow from the intersection--union principle \citep{BergerHsu1996}, and \citet{Lakens2017} gives a modern primer for experimental scientists.
Exact power and sample-size calculations for TOST are derived by \citet{Shieh2016} for the pooled test and by \citet{JanShieh2017} for its heteroscedastic Welch version; Section~\ref{sec:power} benchmarks our design formula against the former.
For ROC analysis specifically, \citet{Liu2006} constructed TOST-type equivalence and non-inferiority tests for the paired AUROCs of two diagnostic procedures evaluated on the same subjects.
Their setting is a single fixed pair of scoring rules, so the only uncertainty is at the item level; in machine-learning comparisons the scoring rule itself is random, since retraining changes it, and the analysis of \citet{Liu2006} corresponds to the inner level of our model with the seed level absent.
A recent tutorial extends this line with simulation-based power analyses for equivalence testing of ROC curves \citep{Riesthuis2025}, still within a single sampling level.
A structurally close hierarchy has long existed in radiology: multi-reader multi-case (MRMC) ROC analysis treats readers and cases as random effects, through a reader~$\times$~case decomposition of jackknife pseudovalues \citep{DorfmanBerbaumMetz1992} or an analysis of variance with case-induced correlated errors \citep{ObuchowskiRockette1995}, and provides the corresponding power and sample-size methodology \citep{Hillis2011}.
\citet{Chen2012} extended this framework to non-inferiority and equivalence hypotheses, with the corresponding power calculation and study sizing, so the MRMC line already combines a two-level ROC hierarchy, an equivalence layer and sample-size methodology.
Our seed level plays the role of the reader level; what the MRMC line does not contain is the reinterpretation of the outer level as the retraining randomness of a single learning method.
The present work adapts the design problem to that setting, makes the conditional estimand of a shared benchmark explicit, and validates the finite-sample behaviour of the protocol in regimes calibrated to machine-learning practice, at two to five runs per method.

A separate line of work quantifies the statistical fragility of machine-learning benchmarking.
\citet{Dietterich1998} initiated the systematic study of hypothesis tests for comparing learning algorithms, \citet{NadeauBengio2003} analysed the variance of generalisation-error estimates across resampled training sets and the inference such estimates permit, and \citet{Demsar2006} recommended the now-standard non-parametric tests for comparing classifiers over several datasets.
\citet{Colas2018} give power analyses for detecting performance differences in deep reinforcement learning, answering ``how many seeds'' for difference testing, and \citet{Card2020} bring power analysis to evaluation in natural language processing, sizing test sets and human-rating studies rather than training runs, and find many common experimental settings to be underpowered.
The magnitude of seed-induced variation is by now documented across domains: \citet{Bouthillier2021} decompose benchmark variance into its sources and study how that variance corrupts the detection of improvements; \citet{Dodge2020} and \citet{Picard2021} measure the spread induced by the training seed alone, together over thousands of runs; \citet{Madaan2024} quantify seed-induced variance in contemporary language-model benchmarks.
\citet{Agarwal2021} advocate interval estimates from a stratified bootstrap over runs for reinforcement-learning benchmarks with few runs, \citet{Miller2024} brings clustered standard errors, paired analyses and power calculations to language-model evaluations, and \citet{Longjohn2025} treat uncertainty quantification for aggregate benchmark metrics.
Closest in structure to the resampling scheme of Section~\ref{sec:model:boot} are the MultiBERTs of \citet{Sellam2022}: $25$ pretraining replicas of BERT together with a Multi-Bootstrap that resamples pretraining seeds and evaluation examples jointly to place confidence intervals on differences between models; their per-example predictions are one of the two data sources of Section~\ref{sec:case2}, the other being the Model Zoos of \citet{Schuerholt2022}, populations of up to a thousand independently trained networks released as checkpoints.
All of these target the detection of differences or the honest reporting of uncertainty.
The Bayesian line of \citet{Benavoli2017} comes closest to the certification question: the region of practical equivalence (ROPE) they adopt answers it through posterior mass rather than through a size-controlled test, and without a run-count design formula.
Where TOST itself appears in the machine-learning literature, it is applied at the annotation or item level without a run-level hierarchy, for instance to test whether a large language model is interchangeable with human annotators \citep{HeLeng2025}.

We are not aware of prior work that places frequentist equivalence certification on the run-level hierarchy of stochastic training.
The closest combination of ingredients is the MRMC equivalence methodology \citep{Chen2012}, built for human readers and fixed diagnostic modalities; the remaining precedents supply single ingredients: the equivalence layer without a run-level hierarchy, at the item level \citep{Liu2006,Riesthuis2025} or at the level of cross-validation folds and datasets \citep{Benavoli2017}, the power analysis for difference testing over runs \citep{Colas2018} or over test items \citep{Card2020}, or the joint resampling of pretraining seeds and evaluation examples without an equivalence layer or a design formula \citep{Sellam2022}.
The formal ingredients themselves are deliberately elementary consequences of classical results; the contribution we claim is not the first combination of hierarchy, equivalence and sizing but their transfer to training runs as the experimental unit, with the conditional benchmark estimand made explicit, a transparent Welch-TOST on per-run AUROCs, and a finite-sample validation aimed at the very small run counts of machine-learning experiments.

\section{Statistical model and equivalence testing procedure}
\label{sec:model}

\subsection{Two-level model}
\label{sec:model:twolevel}

Two training methods $g \in \{A, B\}$ are compared on a common detection task.
Each method is trained $n_g$ times under identical settings except for the random seed, so the runs $s = 1, \dots, n_g$ are independent replicates of the same training procedure.
Every run assigns real-valued scores to a test collection of $m_0$ in-distribution (ID) and $m_1$ out-of-distribution (OOD) items (images, in the case study), and the quality of the run is summarised by the AUROC.

Let $X_1, \dots, X_{m_0}$ denote the ID scores and $Y_1, \dots, Y_{m_1}$ the OOD scores produced by run $(g, s)$.
The test items are drawn i.i.d.\ from their respective populations, the ID and OOD samples independent of each other and of the trained runs, so that conditionally on the run the scores are i.i.d.\ within each sample.
The empirical AUROC is the normalised Mann--Whitney statistic \citep{MannWhitney1947,Bamber1975}
\begin{equation}
\hat\theta_{g,s} = \frac{1}{m_0 m_1} \sum_{i=1}^{m_0} \sum_{j=1}^{m_1}
  \bigl(\mathbf{1}\{X_i < Y_j\} + \tfrac{1}{2} \mathbf{1}\{X_i = Y_j\}\bigr),
\label{eq:auroc}
\end{equation}
an unbiased estimator of the population AUROC of that run, $\theta_{g,s} = \Pr(X < Y) + \tfrac{1}{2}\Pr(X = Y)$.

Retraining with a new seed changes the learned scoring function, so $\theta_{g,s}$ itself varies across runs.
We adopt the two-level (seed~$\times$~item) model
\begin{equation}
\theta_{g,s} = \mu_g + a_{g,s}, \qquad
a_{g,s} \ \text{i.i.d.} \ (0, \sigma_g^2),
\label{eq:model}
\end{equation}
where $\mu_g$ is the expected AUROC of method $g$ and $\sigma_g^2$, the variance of the zero-mean seed effects $a_{g,s}$, is its between-run (seed-level) variance.
The law of the seed effects is left unspecified beyond these two moments, which is all that the variance results of this section use; the Welch-TOST and the power approximation additionally lean on approximate normality of the run means (Sections~\ref{sec:power} and~\ref{sec:discussion}); since $\theta_{g,s}$ is a probability, the law is understood to be supported within $[-\mu_g,\, 1 - \mu_g]$.
Conditionally on the run, that is, on the learned scoring function, which determines more than the value of $\theta_{g,s}$ alone, $\hat\theta_{g,s}$ fluctuates around $\theta_{g,s}$ with the image-level variance $v_{g,s} = \operatorname{Var}\bigl(\hat\theta_{g,s} \mid \mathcal{F}_{g,s}\bigr)$, where $\mathcal{F}_{g,s}$ denotes the trained run; it is estimated nonparametrically from placement values (the normalised ranks of each item within the opposite sample) by the method of \citet{DeLong1988}, in the rank-based implementation of \citet{SunXu2014}.
Write $\bar v_g = \mathbb{E}\, v_{g,s}$.

The estimand is the method difference $\Delta = \mu_A - \mu_B$, estimated by $\hat\Delta = \bar\theta_A - \bar\theta_B$ with $\bar\theta_g = n_g^{-1} \sum_{s} \hat\theta_{g,s}$.

\begin{proposition}[variance decomposition]
\label{prop:variance}
Under model~\eqref{eq:model} and the sampling assumptions above, with every run evaluated on its own test sample drawn independently across runs,
\begin{equation}
\operatorname{Var} \hat\theta_{g,s} = \sigma_g^2 + \bar v_g,
\qquad
\operatorname{Var} \hat\Delta
  = \sum_{g \in \{A, B\}} \frac{\sigma_g^2 + \bar v_g}{n_g}.
\label{eq:vardecomp}
\end{equation}
\end{proposition}

\begin{proof}
Since~\eqref{eq:auroc} is unbiased given the run, $\mathbb{E}[\hat\theta_{g,s} \mid \mathcal{F}_{g,s}] = \theta_{g,s}$, and the law of total variance yields
$\operatorname{Var} \hat\theta_{g,s}
 = \mathbb{E}\bigl[\operatorname{Var}(\hat\theta_{g,s} \mid \mathcal{F}_{g,s})\bigr]
 + \operatorname{Var}\bigl(\mathbb{E}[\hat\theta_{g,s} \mid \mathcal{F}_{g,s}]\bigr)
 = \bar v_g + \sigma_g^2$.
Runs are independent within and between methods, whence the second identity.
\end{proof}

Proposition~\ref{prop:variance} has two design consequences.
First, with $n_g \ge 2$ runs per method, the seed-level sample variance $s_g^2 = (n_g - 1)^{-1} \sum_s (\hat\theta_{g,s} - \bar\theta_g)^2$ estimates the \emph{whole} per-run variance $\sigma_g^2 + \bar v_g$ without bias, so inference that treats per-run AUROCs as i.i.d.\ observations automatically carries the image-level uncertainty and requires no explicit decomposition.
Second, enlarging the test sample shrinks only $\bar v_g$, which is of order $m_0^{-1} + m_1^{-1}$ \citep{HanleyMcNeil1982}: the seed-level floor $\sigma_g^2$ remains, and once $\bar v_g \ll \sigma_g^2$ the precision of $\hat\Delta$ is governed by the number of training runs alone.

In practice all runs are usually scored on the \emph{same} test sample rather than on independent draws, so the seed and image factors are crossed rather than nested.
The design result for this case needs nothing beyond the sampling model of the runs itself, because conditioning on the sample removes the image-sampling level altogether.

\begin{proposition}[shared test set, conditional estimand]
\label{prop:shared}
Let all runs be scored on one common test sample $T$, drawn independently of the runs, and let the runs be, as in Section~\ref{sec:model:twolevel}, independent within and between methods and identically distributed within each method.
Write $\tau_g^2(T) = \operatorname{Var}\bigl(\hat\theta_{g,s} \mid T\bigr)$ for the conditional per-run variance and $\Delta(T) = \mathbb{E}\bigl[\hat\Delta \mid T\bigr]$ for the conditional estimand.
Then
\begin{equation}
\operatorname{Var}\bigl(\hat\Delta \mid T\bigr)
  = \sum_{g} \frac{\tau_g^2(T)}{n_g},
\quad
\mathbb{E}\bigl[s_g^2 \mid T\bigr] = \tau_g^2(T),
\label{eq:sharedvar}
\end{equation}
and the seed-level procedures of this section apply verbatim to $\Delta(T)$, with $s_g^2$ an unbiased plug-in for the \emph{whole} conditional per-run variance.
\end{proposition}

\begin{proof}
Given $T$, the runs of method $g$ are independent replicates of one experiment (train, then score the fixed sample), so the $\hat\theta_{g,s}$ are i.i.d.\ with mean $\mathbb{E}[\hat\theta_{g,s} \mid T]$ and variance $\tau_g^2(T)$; both identities are the standard i.i.d.\ facts applied conditionally.
\end{proof}

\begin{remark}[the two estimands]
\label{rem:estimand}
The conditional estimand $\Delta(T)$, the method difference on this test collection, is usually the question actually asked in benchmark comparisons, and it is the reading we adopt for the case study; note that seed effects and residual image effects are then not separated, since $s_g^2$ captures their conditional total.
The unconditional $\Delta$ of model~\eqref{eq:model} is a different target.
Under the additive decomposition $\hat\theta_{g,s} = \theta_{g,s} + c(T) + \varepsilon_{g,s}$, with a sample effect $c(T)$ shared by all runs of both methods and run-specific residuals satisfying both $\mathbb{E}[\varepsilon_{g,s} \mid \mathcal{F}_{g,s}] = 0$ (no residual run effect) and $\mathbb{E}[\varepsilon_{g,s} \mid T] = 0$ (no residual sample effect), the shift $c(T)$ cancels in $\hat\Delta$ and the residual means vanish given $T$, so $\Delta(T) = \Delta$ and the two readings coincide.
If instead a method--sample interaction $d_g(T)$ is present, $\Delta(T) = \Delta + d_A(T) - d_B(T)$: the interaction does not decrease in the number of runs and is invisible to $s_g^2$, so a single shared sample identifies $\Delta(T)$ but cannot separate $\Delta$ from it.
What recovering $\Delta$ then requires depends on which ``unconditional'' is meant, and two levels must be kept apart.
When the items of $T$ are themselves an i.i.d.\ sample from a fixed population $P$, the $\Delta$ of model~\eqref{eq:model} is the difference over $P$, the realised interaction is item-sampling noise that averages out over redraws of $T$, and its contribution to the uncertainty of $\hat\Delta$ is recoverable from the one sample at hand: by item resampling, by the DeLong estimator, or, by design, by drawing a fresh sample from $P$ for every run (the regime of Proposition~\ref{prop:variance}).
The choice of the benchmark itself, of $P$ with its idiosyncrasies relative to the intended domain, is a level that neither resampling nor fresh draws from the same $P$ reach: only several benchmarks, or an explicit model of how benchmarks arise, speak to it.
The between-run correlation of per-item scores in the case-study data ($0.79$--$0.92$, Section~\ref{sec:case}) indicates substantial shared image structure, though it does not by itself identify the AUROC-level components.
\end{remark}

\subsection{Seed-level equivalence test}
\label{sec:model:tost}

A claim of practical equivalence must invert the usual testing logic: the null hypothesis is a relevant difference,
\begin{equation}
H_0\colon |\Delta| \ge \delta
\quad \text{against} \quad
H_1\colon |\Delta| < \delta,
\label{eq:hypotheses}
\end{equation}
where the margin $\delta$ is fixed before the experiment.
Schuirmann's two one-sided tests \citep{Schuirmann1987,Lakens2017} reject $H_0$ at level $\alpha$ when both one-sided nulls $\Delta \le -\delta$ and $\Delta \ge \delta$ are rejected at level $\alpha$.
With per-run AUROCs as observations the natural statistics are those of \citet{Welch1947}
\begin{equation}
T_{-} = \frac{\hat\Delta + \delta}{\widehat{\mathrm{SE}}},
\qquad
T_{+} = \frac{\hat\Delta - \delta}{\widehat{\mathrm{SE}}},
\label{eq:welch}
\end{equation}
with $\widehat{\mathrm{SE}}^2 = s_A^2/n_A + s_B^2/n_B$, referred to the $t$ distribution with the Welch--Satterthwaite degrees of freedom \citep{Satterthwaite1946}
\begin{equation}
\hat\nu = \frac{\bigl(s_A^2/n_A + s_B^2/n_B\bigr)^2}%
  {\dfrac{s_A^4}{n_A^2 (n_A - 1)} + \dfrac{s_B^4}{n_B^2 (n_B - 1)}}.
\label{eq:df}
\end{equation}
The lower null $\Delta \le -\delta$ is rejected for large $T_-$ ($p_- = 1 - F_{\hat\nu}(T_-)$) and the upper null $\Delta \ge \delta$ for small $T_+$ ($p_+ = F_{\hat\nu}(T_+)$), where $F_{\hat\nu}$ is the central $t$ distribution function; the TOST $p$-value is $p = \max(p_-, p_+)$.
As an intersection--union test, TOST needs no multiplicity correction and controls the size at $\alpha$ whenever each one-sided component holds its level \citep{BergerHsu1996}; the Welch components do so only approximately, since the Behrens--Fisher problem admits no exact solution of the standard $t$ form, and the realised size is therefore verified by simulation in Section~\ref{sec:sim}.

\begin{proposition}[confidence-interval duality]
\label{prop:duality}
TOST rejects $H_0\colon |\Delta| \ge \delta$ at level $\alpha$ if and only if the two-sided $(1 - 2\alpha)$ confidence interval $\hat\Delta \pm t_{1-\alpha,\hat\nu}\, \widehat{\mathrm{SE}}$ lies inside $(-\delta, \delta)$.
\end{proposition}

\begin{proof}
$p_- < \alpha \iff T_- > t_{1-\alpha,\hat\nu} \iff \hat\Delta - t_{1-\alpha,\hat\nu} \widehat{\mathrm{SE}} > -\delta$, and $p_+ < \alpha \iff T_+ < -t_{1-\alpha,\hat\nu} \iff \hat\Delta + t_{1-\alpha,\hat\nu} \widehat{\mathrm{SE}} < \delta$.
Their conjunction is exactly the stated inclusion.
\end{proof}

Reporting the $(1 - 2\alpha)$ interval together with the margin is the form we recommend: it communicates the decision, its sensitivity to the choice of $\delta$, and the direction of any residual difference at once.
The duality is verified numerically, over randomised designs, in the test suite of our implementation.

\subsection{Two-level bootstrap}
\label{sec:model:boot}

A macro-average over several OOD test sets is still a per-run scalar, so the seed-level test above applies to it directly and remains the operative instrument of certification.
Bootstrap resampling enters in a supporting role: it is the de facto practice for reporting the uncertainty of composite statistics (the source study of Section~\ref{sec:case} reports exactly such an interval), and it propagates the item-level uncertainty of composite per-run statistics, such as macro-averages over several test sets, without a variance formula.
We analyse the two-level, crossed resampling scheme that mirrors the design of the experiment.
At the outer level, runs (seeds) are resampled with replacement within each method; at the inner level, one resample of ID item indices and one per OOD set are drawn per bootstrap iteration and shared by all selected runs of both methods, replaying the paired evaluation on a common test sample.
The percentile interval of the resulting bootstrap distribution of $\hat\Delta$ is reported \citep{DavisonHinkley1997}.
We use it for uncertainty reporting rather than for the equivalence decision itself; a bootstrap TOST would require the percentile interval at level $1 - 2\alpha$ (Proposition~\ref{prop:duality}), and the coverage caveats below apply to it with full force.
The two instruments also address different readings of the estimand: the seed-level test is conditional on the realised test collection $T$ (Remark~\ref{rem:estimand}), whereas the item resample makes the bootstrap interval speak to the difference unconditional over item sampling, the recoverable level of Remark~\ref{rem:estimand}, and the coverage study of Section~\ref{sec:sim} evaluates it against that target.

\begin{remark}[domain of validity]
\label{rem:bootvalidity}
Resampling hierarchical data is delicate: the consistency of a scheme depends on the model assumed for the clustered data \citep{FieldWelsh2007}, and which levels to resample is a design decision \citep[Section~3.8]{DavisonHinkley1997}.
We therefore do not claim general consistency of the scheme above; its finite-sample coverage is mapped by simulation in Section~\ref{sec:sim}.
Three regimes emerge.
When the test set is shared and the seed-level variance is comparable to the residual image-level one, coverage is close to nominal.
When the run-specific image-level residual dominates, the scheme is conservative: independent resampling of two crossed factors counts that residual more than once, through each factor and through their product, so that for a linear statistic the corresponding interaction variance enters about three times in the limit \citep{McCullagh2000,Owen2007,OwenEckles2012}; for the AUROC we take this as indicating the direction of the bias and measure its size by simulation.
Conversely, with very few runs and a dominant seed level the outer resample is anti-conservative, since the bootstrap of a mean of $n$ values understates its variance by the factor $(n - 1)/n$, which is severe at $n = 2$--$3$.
In the limit of a negligible item level the $95\%$ percentile interval at $n = 2$ reduces to the interval from the smallest to the largest difference between a run of $A$ and a run of $B$; it covers a true difference of zero exactly when the two within-method ranges overlap, an event of probability $2/3$ for the ideal percentile interval under a common continuous run-level law.
Both deviations are quantified in Section~\ref{sec:sim}.
\end{remark}

\section{Power analysis and the required number of runs}
\label{sec:power}

With the test fixed, its fate is decided by the number of runs it is given, which is the design question of the introduction.
Write $\tau_g^2 = \sigma_g^2 + \bar v_g$ for the per-run variance of method $g$ (Proposition~\ref{prop:variance}; under a shared test set $\tau_g^2$ is read as the conditional per-run variance of Proposition~\ref{prop:shared}, which $s_g^2$ estimates directly).
We consider the equal-allocation design $n_A = n_B = n$ with a common $\tau^2 = \tau_A^2 = \tau_B^2$, the case of practical interest; then
\begin{equation}
\mathrm{SE}(n) = \sqrt{\frac{2 \tau^2}{n}}
             = \sqrt{\frac{2 \bigl(\sigma^2 + \bar v\bigr)}{n}}.
\label{eq:se}
\end{equation}

\begin{proposition}[approximate power of TOST]
\label{prop:power}
Let $\hat\Delta \sim \mathcal{N}\bigl(\Delta, \mathrm{SE}^2(n)\bigr)$ with $\mathrm{SE}(n)$ treated as known.
Then the probability that TOST at level $\alpha$ with margin $\delta$ declares equivalence is
\begin{equation}
\pi(\Delta)
  = \Phi\!\Bigl(\tfrac{\delta-\Delta}{\mathrm{SE}(n)} - z_{1-\alpha}\Bigr)
  + \Phi\!\Bigl(\tfrac{\delta+\Delta}{\mathrm{SE}(n)} - z_{1-\alpha}\Bigr) - 1,
\label{eq:power}
\end{equation}
truncated below at zero, where $\Phi$ is the standard normal distribution function and $z_{1-\alpha}$ its quantile.
\end{proposition}

\begin{proof}
By Proposition~\ref{prop:duality}, with $z_{1-\alpha}$ in place of the $t$ quantile, equivalence is declared if and only if $-\delta + z_{1-\alpha} \mathrm{SE}(n) < \hat\Delta < \delta - z_{1-\alpha} \mathrm{SE}(n)$.
Since $\hat\Delta \sim \mathcal{N}(\Delta, \mathrm{SE}^2(n))$, this probability equals
$\Phi\bigl(\frac{\delta-\Delta}{\mathrm{SE}(n)} - z_{1-\alpha}\bigr)
 - \Phi\bigl(-\frac{\delta+\Delta}{\mathrm{SE}(n)} + z_{1-\alpha}\bigr)$,
which is~\eqref{eq:power}; the rejection region, hence the power, is empty when $\delta \le z_{1-\alpha} \mathrm{SE}(n)$.
\end{proof}

\begin{corollary}[design formula]
\label{cor:n}
For true difference $\Delta = 0$, target power $\pi^{*}$ and margin $\delta$, the smallest sufficient number of runs per method satisfies
\begin{equation}
n \;\ge\; \frac{2\bigl(\sigma^2 + \bar v\bigr)
  \bigl(z_{1-\alpha} + z_{(1+\pi^{*})/2}\bigr)^{2}}{\delta^{2}}.
\label{eq:n}
\end{equation}
\end{corollary}

\begin{proof}
At $\Delta = 0$ formula~\eqref{eq:power} reads $\pi(0) = 2\Phi\bigl(\delta/\mathrm{SE}(n) - z_{1-\alpha}\bigr) - 1$; setting $\pi(0) \ge \pi^{*}$ and solving for $n$ through~\eqref{eq:se} gives~\eqref{eq:n}.
\end{proof}

For the small $n$ typical of machine-learning experiments we sharpen the normal approximation by a shifted-$t$ heuristic: the quantile and distribution function in~\eqref{eq:power} are replaced by their central-$t$ counterparts,
\begin{equation}
\pi_t(\Delta)
  = F_{\nu}\!\bigl(\tfrac{\delta-\Delta}{\mathrm{SE}(n)} - t_{1-\alpha,\nu}\bigr)
  + F_{\nu}\!\bigl(\tfrac{\delta+\Delta}{\mathrm{SE}(n)} - t_{1-\alpha,\nu}\bigr) - 1,
\label{eq:powert}
\end{equation}
truncated below at zero, with $\nu = 2(n-1)$; the design formula~\eqref{eq:n} is then solved by iterating over $n$, with $\nu$ updated at each step.
This is a heuristic, not an exact $t$-theory result: under the alternative the TOST statistics follow noncentral-$t$ laws with a shared random denominator rather than a shifted central $t$, and $\nu = 2(n-1)$ is a design surrogate for the Welch--Satterthwaite value~\eqref{eq:df} realised in the analysis.
Exact power for the pooled, equal-variance normal model follows from the bivariate noncentral $t$ with a common denominator \citep{Owen1965,Phillips1990}; \citet{Shieh2016} computes it as a one-dimensional integral over the pooled variance, and \citet{JanShieh2017} give nearly exact power and sample-size algorithms for the heteroscedastic Welch-TOST.
We implemented the pooled integral of \citet{Shieh2016} next to the heuristic and evaluated both at every cell of the simulation grid of Section~\ref{sec:sim} and at every run count printed in Tables~\ref{tab:design}--\ref{tab:case2pairs} and Section~\ref{sec:case2:pairs}; the comparison is released with the code.
Within the equivalence region and at its boundary ($|\Delta| \le \delta$) the heuristic never exceeds the exact power; outside the region both lie far below $\alpha$ and the heuristic exceeds the exact value by at most $0.003$.
In the design-relevant range of Remark~\ref{rem:powervalidity} the heuristic falls short of the exact power by at most $0.045$ at $n = 5$, $0.018$ at $n = 8$, $0.014$ at $n = 12$ and $0.006$ at $n = 20$; in low-power cells (exact power below $0.3$) the shortfall grows to $0.08$, and at $n = 2$ and $n = 3$ the heuristic reports zero where the exact power reaches $0.31$ and $0.21$.
Inverted for the run count, the two calculations agree in every printed cell except five, where the exact count is one run lower: $n^{*}(90\%)$ at $\delta = 0.015$ in Table~\ref{tab:design}, RTE and MNLI in Table~\ref{tab:case2levels}, and the SST-2 and RTE learning-rate rows of Table~\ref{tab:case2pairs}.
The exact calculation refers to the pooled test, whereas the analysis uses the Welch test, whose random degrees of freedom never exceed $2(n-1)$: on normal run means with a common variance the Welch-TOST realises less power than the pooled one, by up to $0.027$ at $n = 5$, $0.007$ at $n = 8$, $0.004$ at $n = 12$ and $0.003$ at $n = 20$ in the same range, and by up to $0.21$ and $0.09$ at $n = 2$ and $n = 3$ (Monte Carlo with $2 \cdot 10^{5}$ replications per cell, released with the comparison).
In the design-relevant range the heuristic lies between $-0.028$ and $+0.012$ of that Welch power; in low-power cells its shortfall again reaches $0.08$.
We therefore keep the heuristic as the design formula: it is explicit at a given $n$ and transparent, conservative with respect to the exact pooled power, and in the design-relevant range its remaining error is of the size of the pooled-to-Welch gap; the implementation reports all three values, and Section~\ref{sec:sim} validates the heuristic against the realised two-level process.
For a non-zero anticipated difference the same inversion applies to~\eqref{eq:powert}; as power falls when $|\Delta|$ approaches $\delta$, sizing at $\Delta = 0$ yields the smallest $n$.
The equal-variance assumption is also easy to relax: with anticipated per-run variances $\tau_A^2 \ne \tau_B^2$ and equal allocation, $\mathrm{SE}^2(n) = (\tau_A^2 + \tau_B^2)/n$, so~\eqref{eq:n} applies with $2(\sigma^2 + \bar v)$ read as $\tau_A^2 + \tau_B^2$, while the degrees of freedom drop from $2(n - 1)$ to the anticipated Welch--Satterthwaite value $(n - 1)(\tau_A^2 + \tau_B^2)^2 / (\tau_A^4 + \tau_B^4)$, which lowers the power slightly through the $t$ quantile.
Section~\ref{sec:case} quantifies the effect on the case-study estimates.

\begin{figure}[t]
\centering
\includegraphics[width=119mm]{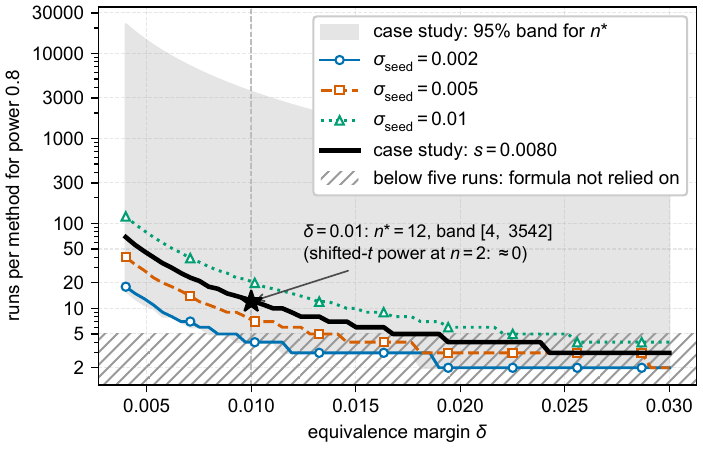}
\caption{Runs per method required for $80\%$ TOST power at $\alpha = 0.05$ and true $\Delta = 0$, as a function of the equivalence margin $\delta$ (logarithmic vertical axis).
The three thin curves with open markers show the design formula~\eqref{eq:n} (shifted-$t$ version) for three seed-level dispersions with the image-level term of the simulation study at $\theta = 0.75$ ($m_0 = m_1 = 10^4$): solid line with circles for $\sigma_{\mathrm{seed}} = 0.002$, dashed line with squares for $\sigma_{\mathrm{seed}} = 0.005$, dotted line with triangles for $\sigma_{\mathrm{seed}} = 0.01$; the thick solid curve is the case study of Section~\ref{sec:case} ($s = 0.0080$, image-level term absorbed into $s^2$).
The star marks $\delta = 0.01$, where $n^{*} = 12$.
The hatched band marks the region below five runs per method, in which we do not rely on the formula; values falling there are raised to five (Remark~\ref{rem:powervalidity}).
The shaded band is the approximate $95\%$ interval for the case-study curve induced by the scaled-$\chi^2$ (Satterthwaite) uncertainty of the plug-in $s^2$ ($\hat\nu \approx 1.2$, Section~\ref{sec:case}); at $\delta = 0.01$ it spans $4$ to $3542$ runs per method}
\label{fig:seedsneeded}
\end{figure}

\begin{remark}[domain of validity]
\label{rem:powervalidity}
In the actual test $\widehat{\mathrm{SE}}$ is random, while Proposition~\ref{prop:power} treats it as known.
In the design-relevant range, approximate power at least $0.2$ and $n \ge 5$, the approximation is nevertheless accurate: in the simulations of Section~\ref{sec:sim} its absolute error does not exceed $0.025$ (median $0.006$ over the grid), and it nowhere overstates power by more than $0.02$, within Monte-Carlo resolution; under the real seed law of Section~\ref{sec:case2} the overstatement reaches $0.03$ (Section~\ref{sec:case2:formula}).
Outside it the approximation degrades where the deterministic power is small, that is, when $\delta$ is close to $z_{1-\alpha} \mathrm{SE}(n)$ and rejections come mainly from realisations that underestimate $\widehat{\mathrm{SE}}$, and it errs on the conservative side, by up to $0.07$ in near-zero-power cells at $n \ge 5$.
The comparison above separates these empirical errors into three parts: the shortfall of the heuristic against the exact pooled-model power, which is conservative at every cell with $|\Delta| \le \delta$; the lower power of the Welch test against the pooled one on normal run means, up to $0.027$ at $n = 5$ in the design-relevant range; and the departure of the realised two-level AUROC process from normal run means under the Welch test, which stays within Monte-Carlo resolution at all $108$ cells inside the margin (at most $0.020$ and median $0.0045$ in the design-relevant range; a two-proportion test of each realised rate against its Welch-normal value gives $|z| \le 2.6$ over the $108$ cells).
At $n = 2$ ($\nu = 2$) it predicts $\approx 0$ where the realised rate at $\Delta = 0$ reaches $0.18$.
At $n = 3$ the error runs in both directions, from $-0.17$ to $+0.05$ across the grid, so below $n = 5$ we do not rely on the formula at all.
The realised size at both boundaries $\Delta = \pm\delta$ stays at or near $\alpha$, consistent with the intersection--union argument.
The largest value over the $108$ boundary cells is $0.063$ (Monte-Carlo standard error $0.005$), an excursion of the size expected for the maximum over that many cells and no evidence of systematic inflation.
Both patterns persist when the Gaussian seed effects are replaced by skewed or heavy-tailed laws (Section~\ref{sec:sim}).
Design guidance based on~\eqref{eq:n} thus errs on the safe side where it is imprecise; recommendations that land below five runs should be raised to five.
\end{remark}

Fig.~\ref{fig:seedsneeded} turns Corollary~\ref{cor:n} into design curves: the number of runs required for $80\%$ power against margin $\delta$, for representative seed-level dispersions.
The steepness of $n^{*}(\delta) \propto \delta^{-2}$ near small margins is the central practical message: certifying a tight equivalence margin is expensive in training runs, and once the image-level term is subordinate to the seed-level one, no test-set enlargement can remove the seed-level floor (Proposition~\ref{prop:variance}).

\section{Simulation study}
\label{sec:sim}

Both theoretical claims, the power approximation of Proposition~\ref{prop:power} and the coverage behaviour of the two-level bootstrap (Remark~\ref{rem:bootvalidity}), are checked against full simulations of the two-level model.
Scores are generated from the binormal model: within run $(g,s)$, ID scores are $\mathcal{N}(0, 1)$ and OOD scores $\mathcal{N}(m_{g,s}, 1)$, where $m_{g,s} = \sqrt{2}\, \Phi^{-1}(\theta_{g,s})$ links the location shift to the run-level AUROC, and $\theta_{g,s}$ follows model~\eqref{eq:model} with Gaussian seed effects.
The true difference is assigned to method $A$ ($\mu_A = \theta + \Delta$, $\mu_B = \theta$); on a shared test sample each score mixes a common item effect and run noise with weights $\sqrt{\rho}$ and $\sqrt{1 - \rho}$, which fixes the cross-run correlation of same-item scores at $\rho$ without changing the marginal laws; the root seeds of all grids are fixed in the simulation scripts.
Uncertainties quoted as $\pm$ in the text are single Monte-Carlo standard errors; the figures show $\pm 3$ standard errors.
The image-level variances implied by this generator are well approximated by the classical formula of \citet{HanleyMcNeil1982}, which rests on an exponential score model: it agrees with the binormal value to within about $2\%$ at $\theta \le 0.75$ but exceeds it by $13\%$ at $\theta = 0.90$; at $m_0 = m_1 = 10^4$ the binormal values are $\bar v = 1.65 \cdot 10^{-5}$, $1.16 \cdot 10^{-5}$ and $4.5 \cdot 10^{-6}$ for $\theta = 0.55$, $0.75$ and $0.90$.
Draws of $\theta_{g,s}$ outside the domain of the link are rejected and redrawn, which keeps the generator within the support constraint of model~\eqref{eq:model}; under the Gaussian seed effects and the grid parameters below the rejection probability is under $10^{-15}$ per draw, no rejection occurred in the realised grids, and the results therefore coincide with unrestricted sampling.

\emph{Power validation.}
The grid crosses $\sigma \in \{0.002, 0.005, 0.01\}$, $n \in \{2, 3, 5, 8, 12, 20\}$, $\theta \in \{0.55, 0.75, 0.90\}$ and true $\Delta \in \{-0.01, 0, 0.005, 0.01, 0.02\}$ at margin $\delta = 0.01$ and $\alpha = 0.05$, so that the values $\Delta = \pm\delta$ sit on the two boundaries of the composite null, with $2000$ replications per cell and an independent test draw per run (the regime of Proposition~\ref{prop:variance}).
Fig.~\ref{fig:power} compares the empirical rejection rate with the shifted-$t$ approximation; the agreement pattern (accurate in the design-relevant range, conservative in the near-zero-power regime and at $n = 2$, erring in both directions at $n = 3$, size at the boundaries consistent with $\alpha$) is quantified in Remark~\ref{rem:powervalidity}.
The validation is run under independent test draws; under a shared test set the per-run variance changes as in Proposition~\ref{prop:shared}, a change that the plug-in $s^2$ of Section~\ref{sec:case} absorbs, while the shape of the conditional law lies outside the scope of that proposition and is not validated separately here.

\emph{Robustness to the seed-effect law.}
A companion grid replaces the Gaussian seed effects by a centred exponential law (skewness $2$) and by a scaled $t_3$ law (heavy tails), crossing the three laws, a Gaussian anchor included, with $\sigma \in \{0.005, 0.01\}$, $n \in \{2, 3, 5, 12\}$ and $\Delta \in \{-0.01, 0, 0.005, 0.01\}$ at $\theta = 0.75$ ($96$ cells, $2000$ replications each).
Since the exponential and $t_3$ laws are unbounded, draws of $\theta_{g,s}$ are clipped to $(0.505, 0.999)$, so $\sigma$ denotes the pre-clipping scale.
The clip probability peaks at $3 \cdot 10^{-5}$ per draw ($t_3$ at $\sigma = 0.01$), yet because a sizeable part of the variance of $t_3$ resides in its far tails, clipping lowers that variance by $3.4\%$ at $\sigma = 0.01$ and by $1.7\%$ at $\sigma = 0.005$; the exponential law is unaffected at the reported precision.
This is a shift of under $2\%$ in the standard deviation and of under $0.03$ in the predicted power in the affected cells, small against the deviations reported below, and it acts in the conservative direction for the formula, which uses the nominal $\sigma$.
The conclusions transfer: the realised size at the boundaries $\Delta = \pm\delta$ never exceeds $0.053$ under any law, and in the design-relevant range the formula overstates power nowhere beyond Monte-Carlo resolution, its errors growing only in the conservative direction, with understatements of up to $0.05$ under skewness and up to $0.125$ under heavy tails.
On the exponential and $t_3$ grids, design guidance based on the formula therefore errs further on the safe side; the real seed law of Section~\ref{sec:case2:formula} shows that this need not hold for every non-Gaussian law.

\begin{figure}[t]
\centering
\includegraphics[width=119mm]{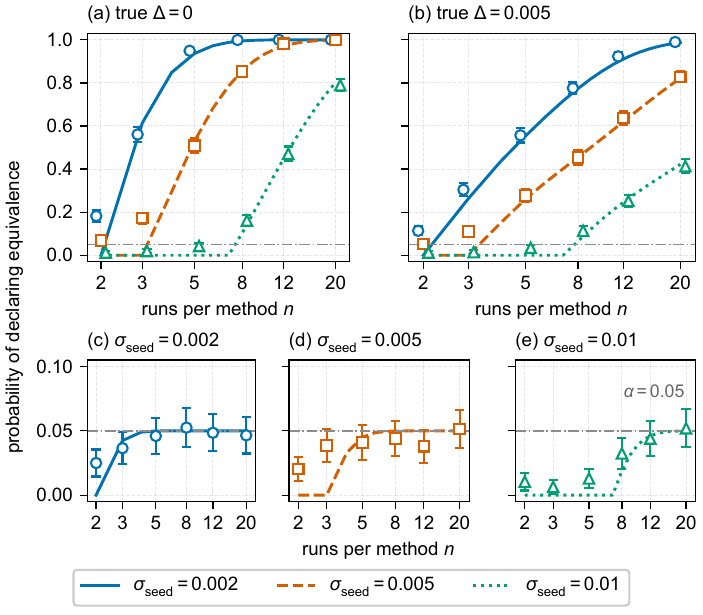}
\caption{Empirical probability that TOST declares equivalence (markers, $\pm 3$ Monte-Carlo standard errors, $2000$ replications per point) against the shifted-$t$ approximation~\eqref{eq:powert} (lines), at $\theta = 0.75$ and margin $\delta = 0.01$.
Panels (a) and (b) show the power at true differences $\Delta = 0$ and $\Delta = 0.005$; panels (c)--(e) show the realised size at the boundary $\Delta = \delta$ for the three seed-level dispersions $\sigma_{\mathrm{seed}} = 0.002$, $0.005$ and $0.01$, one panel each, on an enlarged vertical scale ($0$ to $0.1$).
Line style and marker encode $\sigma_{\mathrm{seed}}$ as in Fig.~\ref{fig:seedsneeded}.
The horizontal axes are logarithmic, and in panels (a) and (b) the markers of the three series are offset slightly so that their error bars stay apart.
The approximation is accurate where the power is practically useful, conservative in the near-zero-power regime, most visibly at $n = 2$, and misses in both directions at $n = 3$ (Remark~\ref{rem:powervalidity}).
The dash-dotted lines mark $\alpha = 0.05$}
\label{fig:power}
\end{figure}

\emph{Interval coverage.}
All runs share one test sample; the baseline configuration has between-run score correlation $\rho = 0.8$, calibrated to the case-study data ($0.79$--$0.92$ across test sets), true difference zero, $\theta = 0.6$, $m_0 = m_1 = 2000$, $B = 400$ bootstrap iterations, $1000$ replications per cell.
The comparison baseline is the naive per-image bootstrap that resamples items only, with all runs fixed: the common practice when a single run per method is available, and a valid interval for the different question of comparing two \emph{fixed} trained models rather than the training methods behind them.
The grid maps the $(n, \sigma)$ plane at $n \in \{2, 3, 5, 8, 12\}$ and $\sigma \in \{0.002, 0.005, 0.01\}$ and probes three further axes ($\rho \in \{0.5, 0.95\}$, a non-zero true difference $\Delta = 0.005$, and $\theta = 0.75$), for $31$ cells in total.
Fig.~\ref{fig:coverage} shows the $(n, \sigma)$ map.
The naive interval collapses whenever the seed level matters and never recovers: at $\sigma = 0.01$ its coverage stays near $0.53$ for every $n$, its width does not react to $\sigma$ at all, and averaging more runs only makes it narrower.
The two-level interval reproduces the three regimes of Remark~\ref{rem:bootvalidity}: conservative under a dominant residual image level ($0.98$--$1.00$ at $\sigma = 0.002$, growing with $n$), close to nominal or mildly conservative when the two variance levels are balanced ($0.96$--$0.99$ at $\sigma = 0.005$), and anti-conservative with few runs and a dominant seed level ($0.835 \pm 0.012$ at $\sigma = 0.01$, $n = 2$), recovering to $0.94$--$0.95$ by $n = 5$--$8$.
The probe axes confirm that the governing quantity is the share of the run-specific residual: raising the score correlation to $\rho = 0.95$ deepens the small-$n$ deficit ($0.725 \pm 0.014$ at $n = 2$, $\sigma = 0.01$, and still $0.897 \pm 0.010$ at $n = 5$), lowering it to $\rho = 0.5$ attenuates it ($0.912 \pm 0.009$ in the same cell), while a non-zero true difference or $\theta = 0.75$ shifts coverage at $n = 2$ by at most about $0.04$ and leaves $n = 5$ essentially unchanged.

\begin{figure}[t]
\centering
\includegraphics[width=119mm]{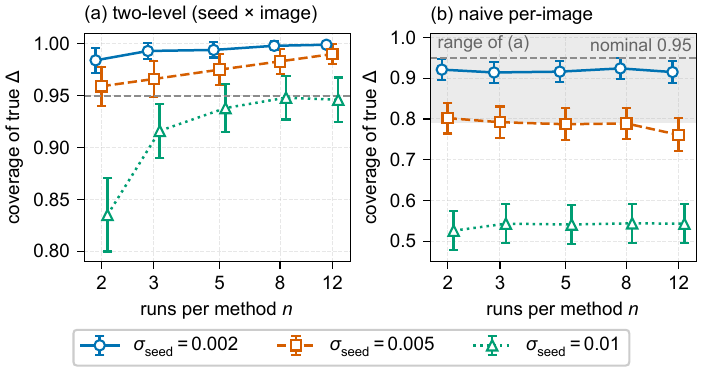}
\caption{Coverage of the true difference $\Delta = 0$ by nominal $95\%$ intervals on a shared test set ($\theta = 0.6$, $m_0 = m_1 = 2000$, $\rho = 0.8$, $B = 400$, $1000$ replications per point, $\pm 3$ Monte-Carlo standard errors) as a function of the number of runs: (a) the two-level bootstrap of Section~\ref{sec:model:boot} against (b) the naive per-image bootstrap that ignores the seed level; panel (a) enlarges the band shaded in panel (b), the dashed horizontal line marks the nominal level $0.95$, the horizontal axis is logarithmic, and the markers are offset slightly so that their error bars stay apart.
Line style and marker encode $\sigma_{\mathrm{seed}}$ as in Fig.~\ref{fig:seedsneeded}.
The naive interval is far below nominal whenever the seed level matters and does not improve with more runs; the two-level interval departs from nominal only in the regimes of Remark~\ref{rem:bootvalidity}}
\label{fig:coverage}
\end{figure}

\section{Case study: comparing two training methods for joint energy-based models}
\label{sec:case}

The protocol is now applied to the data of our experimental study of joint energy-based models \citep{Knopov2026jem}, which compared two training methods on CIFAR-10: the standard stochastic gradient Langevin dynamics (SGLD) sampler and a fixed-noise adaptation of the predictor--corrector (PC) sampler.
In the notation of Section~\ref{sec:model}, $A = \mathrm{PC}$ and $B = \mathrm{SGLD}$, so $\Delta = \mu_{\mathrm{PC}} - \mu_{\mathrm{SGLD}}$.
Each method was trained twice with different random seeds ($n = 2$).
Every run was evaluated at the checkpoint prescribed by the selection rule of the source study, about ten epochs before that run's own divergence point; the run unit therefore comprises training, this checkpoint rule and evaluation, and the design counts below refer to runs evaluated under that rule.
Every run was scored on the shared CIFAR-10 test sample ($m_0 = 10\,000$ ID items) against five OOD test sets: SVHN ($m_1 = 26\,032$), CIFAR-100 ($10\,000$), DTD ($1\,880$), LSUN-R ($10\,000$) and iSUN ($8\,925$).
The per-run summary is the macro-average of the five AUROCs, and the equivalence margin adopted there is $\delta = 0.01$.
The margin is a domain choice (Section~\ref{sec:discussion}), and here it prices the decision at stake: the adapted sampler adds implementation and tuning complexity over SGLD, and a quality gain below one hundredth of AUROC was judged too small to justify adopting it; this is a judgement about practical consequences, not about detectability, since with enough runs far smaller differences become estimable.
The window is symmetric because the claim under test is interchangeability, in the expected macro-average of the five AUROCs under the evaluation protocol of the source study, rather than a one-directional switch: an advantage beyond the margin would justify the extra complexity, while a deficit beyond it would disqualify the sampler as a drop-in replacement.
The case study is retrospective, since in \citet{Knopov2026jem} the margin was adopted at analysis time rather than pre-registered, so we read it as an illustration of the protocol, whose prospective form fixes $\delta$ before data collection.
Table~\ref{tab:runs} lists the per-run values; all numbers below are recomputed from the raw per-item scores with our implementation and reproduce the published ones exactly.
Table~\ref{tab:design} shows how the certification cost responds to tightening or relaxing the margin.
The per-item scores of all four runs are archived on Zenodo with the source study \citep{Knopov2026jem}, \url{https://doi.org/10.5281/zenodo.21872469}.

\begin{table}[t]
\caption{Per-run AUROCs of the case study: two training methods, two seeds each, evaluated at the checkpoints prescribed by the source study, five OOD test sets and their macro-average.
Recomputed from the raw per-item scores of \citet{Knopov2026jem}}
\label{tab:runs}
\centering
\begin{tabular}{lcccccc}
\hline
Run & SVHN & CIFAR-100 & DTD & LSUN-R & iSUN & macro-5 \\
\hline
SGLD, seed 42  & $0.5676$ & $0.6396$ & $0.7250$ & $0.7805$ & $0.7863$ & $0.6998$ \\
SGLD, seed 123 & $0.6056$ & $0.6244$ & $0.7349$ & $0.7688$ & $0.7899$ & $0.7047$ \\
PC, seed 42    & $0.6045$ & $0.6418$ & $0.7374$ & $0.7511$ & $0.7668$ & $0.7003$ \\
PC, seed 123   & $0.6237$ & $0.6391$ & $0.7549$ & $0.7737$ & $0.7859$ & $0.7155$ \\
\hline
\end{tabular}
\end{table}

\emph{Estimates.}
The method means are $0.7022$ (SGLD) and $0.7079$ (PC); the seed-level dispersions are $s_{\mathrm{SGLD}} = 0.0035$ and $s_{\mathrm{PC}} = 0.0107$ (one degree of freedom each), pooling to $s^2 = 6.34 \cdot 10^{-5}$, i.e.\ $s = 0.0080$.
The image-level (DeLong) variance of the macro-average, per-set components plus the cross-set covariance induced by the shared ID sample, is $\bar v = 8.6 \cdot 10^{-6}$ per run, essentially identical across the four runs and only $14\%$ of $s^2$ on the plug-in estimates: this suggests a seed-dominated comparison even before any cancellation, though the ratio inherits the uncertainty of $s^2$ itself.
The between-run correlation of per-item scores ($0.79$--$0.92$) indicates that much of that image structure is in fact shared, though it does not by itself identify the AUROC-level split (Remark~\ref{rem:estimand}).
The split is immaterial for the test: by Proposition~\ref{prop:shared} the pooled $s^2$ estimates the whole conditional per-run variance on this benchmark, whatever the share of residual image noise inside it, and it is used below as the plug-in.
The two methods reuse the seed values $42$ and $123$; the source study analysed the runs as unpaired, and we reproduce that working analysis here (the seeds were not designed as matched blocks, which by itself does not establish independence), with the paired TOST as a sensitivity check: it gives $p = 0.28$ and the same conclusion (incidentally, the between-method score correlation is the same for same-seed pairs, $0.878$, as for different-seed ones, $0.876$).
The collections of Section~\ref{sec:case2:pairs}, by contrast, are matched by seed by construction and are analysed there with the paired TOST as the primary test.

\emph{Equivalence test.}
The seed-level Welch-TOST gives $\hat\Delta = +0.0056$ (PC minus SGLD), $\widehat{\mathrm{SE}} = 0.0080$, $\hat\nu = 1.21$ and $p = 0.334$: equivalence within $\pm 0.01$ cannot be declared.
The $90\%$ interval of Proposition~\ref{prop:duality} is $[-0.032, +0.043]$, whose width is almost four times the full margin window.
The two-level bootstrap interval ($95\%$, $B = 3000$) is $[-0.005, +0.016]$; at $n = 2$ we treat it as illustrative, the Welch interval being the operative one (Section~\ref{sec:discussion}).
By the coverage study (Fig.~\ref{fig:coverage}), an interval obtained in the $\sigma = 0.01$, $n = 2$ regime is materially too narrow: its coverage in the simulated configurations is $0.84$ at $\rho = 0.8$ and $0.73$ at $\rho = 0.95$, and the correlation in these data lies between the two.
A corrected, wider interval would still contain zero and still cross the margin, so the descriptive reading of the data is unchanged: they are compatible both with equivalence and with practically relevant differences in either direction, and certify neither.

\emph{Retrospective power.}
Was the null outcome bad luck or design?
Corollary~\ref{cor:n} answers clearly, though conditionally on the plug-in $s^2$: at $n = 2$ the rejection region of the deterministic approximation is empty, since $\delta = 0.01 < t_{0.95,\,2}\, \mathrm{SE}(2) = 0.023$, and the shifted-$t$ power at $\Delta = 0$ is zero.
The exact pooled-model power of Section~\ref{sec:power} is small but positive at $n = 2$, $0.054$ at $\Delta = 0$ and $0.044$ at $\Delta = \hat\Delta$, because a pair of runs with an unusually small $s^2$ can still certify.
A simulation of the two-level model calibrated to the estimates above ($\theta = 0.705$, $\rho = 0.89$, $\sigma = 0.0079$, $2000$ replications; a single binormal AUROC stands in for the five-set macro-average) realises a rejection rate of $0.026 \pm 0.004$, the conservative regime of Remark~\ref{rem:powervalidity}.
This rate is itself a plug-in quantity: the variance calibration behind it rests on $\hat\nu \approx 1.2$ effective degrees of freedom, so the number is a model-based estimate rather than a property of the design alone.
The deterministic emptiness is conditional on the plug-in as well, but far more stable: at $n = 2$ the region is empty whenever $s^2 > (\delta / t_{0.95,\,2})^2 = 1.2 \cdot 10^{-5}$, which holds across the entire $\chi^2$ interval for the variance reported below.
Under that calibration, whatever the true difference, the certification probability at $n = 2$ stays in the low per~cents, with $\Delta = 0$ the most favourable case, so the failure of the TOST to declare equivalence was determined largely by the design rather than by the data.

\emph{The design answer.}
For $80\%$ power at $\delta = 0.01$ the design formula requires $n^{*} = 12$ runs per method ($15$ for $90\%$), marked by the star in Fig.~\ref{fig:seedsneeded}; at $n = 12$ the calibrated simulation realises a power of $0.811 \pm 0.009$ against the predicted $0.812$ (both the formula and this simulation assume a common per-run variance).
\citet{Knopov2026jem} quotes these counts in its limitations section as order-of-magnitude guidance; their derivation, simulation check and sensitivities are the subject of the present section.
Both numbers also assume exact equality, the most favourable case for the test: if the true difference instead equals the observed $\hat\Delta = 0.0056$, the power of the $n = 12$ design drops to $0.36$, and $80\%$ power at the same margin needs about $42$ runs per method.
A second sensitivity concerns the dispersions, which the two methods appear to violate ($s_{\mathrm{PC}}^2 / s_{\mathrm{SGLD}}^2 \approx 9.5$): if that ratio persisted, the anticipated Welch degrees of freedom of Section~\ref{sec:power} at $n = 12$ would be $13.3$ rather than $22$, the realised power $0.79$ rather than $0.81$, and $n^{*}$ would move from $12$ to $13$, within the order-of-magnitude reading of these numbers.
The ratio itself, however, rests on one degree of freedom per method and is far too uncertain to distinguish equal from unequal true dispersions (a two-sided $F(1,1)$ test gives $p \approx 0.4$).
Table~\ref{tab:design} tabulates $n^{*}$ over margins at $\Delta = 0$: relaxing the margin to $0.02$ brings the cost down to the floor of five runs per method (raw formula values $4$--$5$), while certifying $0.005$ would take $45$.
The exact pooled-model calculation of Section~\ref{sec:power} returns the same counts in every cell except $n^{*}(90\%) = 7$ rather than $8$ at $\delta = 0.015$.

\begin{table}[t]
\caption{Runs per method $n^{*}$ at $\alpha = 0.05$ and true $\Delta = 0$ with the case-study plug-in $s = 0.0080$, from the shifted-$t$ design formula of Corollary~\ref{cor:n} (formula) and from the exact pooled-model power of \citet{Shieh2016} inverted for $n$ (exact).
Entries below $n = 5$ lie outside the simulation-validated range of the formula (Remark~\ref{rem:powervalidity}) and should be read as ``at least five''}
\label{tab:design}
\centering
\begin{tabular*}{\textwidth}{@{\extracolsep{\fill}}ccccc}
\hline
& \multicolumn{2}{c}{$n^{*}$ ($80\%$)} & \multicolumn{2}{c}{$n^{*}$ ($90\%$)} \\
$\delta$ & formula & exact & formula & exact \\
\hline
$0.005$  & $45$ & $45$ & $56$ & $56$ \\
$0.0075$ & $21$ & $21$ & $26$ & $26$ \\
$0.01$   & $12$ & $12$ & $15$ & $15$ \\
$0.015$  & $6$  & $6$  & $8$  & $7$  \\
$0.02$   & $4$  & $4$  & $5$  & $5$  \\
\hline
\end{tabular*}
\end{table}

One caveat is essential for honest use of these numbers.
With two runs per method the pooled $s^2$ itself rests on $\hat\nu \approx 1.2$ effective degrees of freedom: the approximate $\chi^2$-based $95\%$ interval for the true per-run variance, with the pooled $s^2$ treated as a Satterthwaite scaled-$\chi^2$, stretches from $1.4 \cdot 10^{-5}$ to $2.1 \cdot 10^{-2}$, and the corresponding $n^{*}(80\%)$ ranges from $4$ into the thousands (the shaded band of Fig.~\ref{fig:seedsneeded}).
We have not verified the coverage of that Satterthwaite approximation; under normal run means with a common variance the exact two-degree-of-freedom interval is narrower, $1.7 \cdot 10^{-5}$ to $2.5 \cdot 10^{-3}$, and still spans $n^{*}$ from $4$ to about $430$.
The plug-in $n^{*} = 12$ is an order-of-magnitude statement, not a certified sample size.
The practical recommendation for a certification experiment is therefore two-stage: a small pilot ($4$--$6$ runs per method) to sharpen the estimate of $s$ and to bound the plausible true difference, then Corollary~\ref{cor:n} with the updated values.
The pilot sizes the remaining uncertainty itself: at $5$ runs per method the pooled $s^2$ carries $8$ degrees of freedom under a normal, common-variance working model, and the multiplicative $95\%$ uncertainty of $n^{*}$ shrinks from $\times 0.2$--$\times 330$ to $\times 0.5$--$\times 3.7$.
The simplest confirmatory implementation uses the pilot for design only and excludes its runs from the final test; reusing them would make the run count depend on the observed data, an internal-pilot design whose size control must be established separately.
Section~\ref{sec:case2:pairs} carries this procedure through on a real seed law; there, on one collection and two margins, entering the pilot variance as an upper confidence bound rather than as a point estimate, following \citet{Browne1995}, brought the achieved power close to the target, at the price of more runs.

\section{Second case study: real seed laws in two public collections}
\label{sec:case2}

The case study of Section~\ref{sec:case} has the run count that motivated this work and no more; it cannot show how the protocol behaves when runs are plentiful, nor whether real seed effects resemble the laws simulated in Section~\ref{sec:sim}.
Two public collections of independently trained models allow both questions to be asked on real data at no training cost.
The first is the CIFAR-10 seed zoo of \citet{Schuerholt2022}: one small convolutional network in one training configuration, trained from $1000$ random seeds, with a checkpoint saved after every epoch \citep{ModelZoosData2022}; the companion fixed-seed hyperparameter zoo covers $128$ configurations with $10$ seeds each.
The second is the MultiBERTs release of \citet{Sellam2022}: $25$ independently pretrained BERT-base models, fine-tuned with several learning rates and fine-tuning seeds, together with their per-example predictions on the development sets of eight tasks of the General Language Understanding Evaluation benchmark (GLUE), class probabilities for the seven classification tasks, released alongside the checkpoints rather than described in the paper.
We use three of them, chosen to span test-set sizes from a few hundred to ten thousand items: RTE (Recognizing Textual Entailment, $m = 277$ examples), SST-2 (Stanford Sentiment Treebank, $872$) and MNLI (Multi-Genre Natural Language Inference, matched development set, $9\,815$).

Both collections record classification outputs, so the metric is accuracy rather than AUROC.
The seed-level layer of Section~\ref{sec:model} applies unchanged; at the item level the accuracy of a run is a mean of per-item correctness indicators, its variance conditional on the run is $p(1 - p)/m$, and the between-run correlation of correctness plays the role of $\rho$.
For the zoo we recomputed the per-image predictions of the final checkpoint of every run on the CIFAR-10 test set; the resulting accuracies match the values shipped in the release exactly for $999$ of the $1000$ runs and differ by one image in the remaining one.
A run of the zoo is one seed; a run of MultiBERTs is one pretraining seed with one fixed fine-tuning seed, so that the $25$ runs of a configuration are independent and free of the clustering that the $125$-run collection with five fine-tuning seeds per pretraining seed would introduce.
Fixing the fine-tuning seed is empirically benign for the per-run variance at this precision: the per-run variance at the fixed seed lies within $20\%$ of the total variance of the $125$-run collection on all three tasks (Section~\ref{sec:case2:scope}).
All numbers below are produced by the scripts released with this article.

\subsection{Two-level structure}
\label{sec:case2:levels}

Table~\ref{tab:case2levels} summarises the four single-configuration collections.
Its $n^{*}$ column uses the shifted-$t$ formula; the exact pooled-model count of Section~\ref{sec:power} is one lower for RTE ($635$) and MNLI ($2$) and identical for the other two.
Three observations matter for design.
First, real seed laws are neither Gaussian nor alike: the zoo's accuracies are strongly left-skewed with heavy tails, with a minority of seeds at markedly worse solutions (the lowest decile lies $0.02$ below the median and the minimum $0.08$ below it), whereas the three GLUE tasks are at most moderately skewed without marked excess kurtosis; at $25$ runs their shape statistics are compatible with a Gaussian law (Shapiro--Wilk $p \ge 0.13$), in contrast with the zoo's $1000$ runs.
Second, the seed-level dispersion spans more than an order of magnitude across tasks trained with the same recipe: $0.002$ on MNLI, $0.005$ on SST-2 and $0.061$ on RTE, where fine-tuning is unstable \citep{Dodge2020}.
For RTE this puts a margin of one accuracy point out of reach of any realistic budget ($n^{*} = 636$ runs per method), and even a margin of five points needs $27$ runs per method; on such a benchmark equivalence is certifiable only at margins comparable to its own seed dispersion.
Third, the binomial item-level variance exceeds the observed per-run variance on SST-2 and MNLI ($\bar v / s^2 = 3.6$ and $2.9$).
This does not contradict Proposition~\ref{prop:variance}; it is the shared-test-set case of Proposition~\ref{prop:shared}.
The quantity $\bar v$ is the variance over redraws of the test items, which a fixed benchmark never realises, whereas $s^2$ is the conditional per-run variance.
When the runs agree on most items ($82\%$ of the SST-2 items are classified correctly by all $25$ runs, $\hat\rho = 0.70$), their accuracies differ by less than binomial noise would suggest.
The design consequence is concrete: sizing a shared-benchmark comparison with $s^2 + \bar v$ in place of $s^2$ would ask for $20$ runs per method on SST-2 instead of $5$; the plug-in of Section~\ref{sec:power} on a shared test set is $s^2$ alone, and the sizing and the certification then refer to the conditional difference $\Delta(T)$ on that benchmark (Proposition~\ref{prop:shared}).

\begin{table}[t]
\caption{Two-level structure of four single-configuration collections: runs $R$, test items $m$, mean accuracy, seed-level standard deviation $s$, skewness and excess kurtosis of the per-run accuracies, ratio of the mean binomial item-level variance $\bar v$ to $s^2$, between-run correlation of per-item correctness $\hat\rho$, share of items classified correctly by every run (column \emph{always}), and runs per method $n^{*}$ for $80\%$ power at $\delta = 0.01$ from the design formula with the plug-in $s$ (shared test set); values of $n^{*}$ below $5$ are to be read as ``at least five'' (Remark~\ref{rem:powervalidity})}
\label{tab:case2levels}
\centering\footnotesize\setlength{\tabcolsep}{3.5pt}
\begin{tabular}{lrrrrrrrrrr}
\hline
Collection & $R$ & $m$ & mean & $s$ & skew & kurt & $\bar v/s^2$ & $\hat\rho$ & always & $n^{*}$ \\
\hline
CIFAR-10 zoo      & 1000 & 10\,000 & 0.487 & 0.0136 & $-1.80$ & $5.40$ & 0.14 & 0.57 & 0.04 & 33 \\
MultiBERTs, RTE   & 25   & 277     & 0.628 & 0.0609 & $-0.83$ & $0.24$ & 0.22 & 0.39 & 0.00 & 636 \\
MultiBERTs, SST-2 & 25   & 872     & 0.920 & 0.0049 & $-0.43$ & $0.14$ & 3.59 & 0.70 & 0.82 & 5 \\
MultiBERTs, MNLI  & 25   & 9\,815  & 0.843 & 0.0021 & $0.38$  & $-0.33$ & 2.94 & 0.66 & 0.64 & 3 \\
\hline
\end{tabular}
\end{table}

\subsection{The design formula under a real seed law}
\label{sec:case2:formula}

With $1000$ runs of one configuration, the certification probability of the seed-level test at a true difference of zero can be measured directly: $n$ runs are drawn without replacement for each of two disjoint arms, the TOST is applied, and the draw is repeated $5000$ times for each $n \in \{2, 3, 5, 8, 12, 20, 50\}$.
Adding $\pm\delta$ to one arm produces the boundaries of the null exactly, so the realised size is measured the same way.
Fig.~\ref{fig:realseed} shows the outcome at the margin $\delta = 0.01$ of Section~\ref{sec:case}, for which the formula with the plug-in $s = 0.0136$ asks for $n^{*} = 33$ runs, and at $\delta = 0.0173$, the margin at which it asks for $n^{*} = 12$; the same check on the $25$ SST-2 runs is added at $\delta = 0.01$ ($n^{*} = 5$).
Under the zoo's heavy-tailed law the pattern of Remark~\ref{rem:powervalidity} reappears with one qualification.
Below eight runs the formula is conservative, at times substantially ($0.12$ predicted against $0.38$ realised at $n = 5$, $\delta = 0.0173$): small arms rarely contain a poorly converged seed, so the realised standard error is typically smaller than the pooled $s$ implies.
At the tighter margin $\delta = 0.01$ the predicted power stays below $0.2$ up to $n = 12$, and the conservative gap persists that far ($0.07$ predicted against $0.24$ realised at $n = 12$).
Within the design-relevant range of Remark~\ref{rem:powervalidity}, where the approximate power is at least $0.2$, the error stays within $\pm 0.05$ and changes sign.
At the design point $n^{*} = 12$ for $\delta = 0.0173$ the realised rate is $0.794 \pm 0.006$ against the predicted $0.826$, the largest overstatement we observed, consistent with arms of that size catching an outlier often enough to inflate the Welch standard error.
At $n = 20$ and $50$ the formula and the data agree to within $0.03$, and the SST-2 series, whose law is close to Gaussian, agrees from $n = 3$ upward ($0.297$ realised against $0.286$ predicted at $n = 3$, $0.796$ against $0.801$ at $n = 5$).
The realised size stays at or near $\alpha$ throughout: the largest of the $28$ boundary cells of the zoo is $0.055 \pm 0.003$, an excursion of the size expected for a maximum over that many cells, so the size behaves under the real law as the intersection--union argument predicts, as it did in simulation.

\begin{figure}[t]
\centering
\includegraphics[width=119mm]{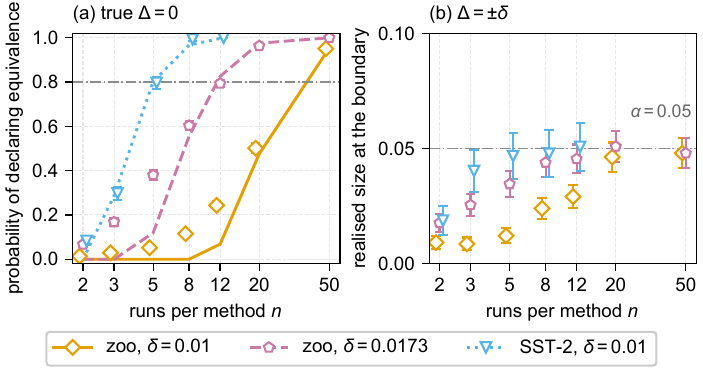}
\caption{The design formula against a real seed law: (a) probability that the seed-level TOST declares equivalence at true $\Delta = 0$ and (b) realised size at the boundaries $\Delta = \pm\delta$ when $n$ runs per method are drawn without replacement from the $1000$-seed CIFAR-10 zoo, at $\delta = 0.01$ ($n^{*} = 33$, solid line with diamonds) and $\delta = 0.0173$ ($n^{*} = 12$, dashed line with pentagons), and from the $25$ SST-2 runs of MultiBERTs at $\delta = 0.01$ ($n^{*} = 5$, dotted line with inverted triangles).
Markers are realised rates with $\pm 3$ Monte-Carlo standard errors (in panel (b) the mean of the two boundary rates), and lines in panel (a) show the shifted-$t$ formula~\eqref{eq:powert} with the plug-in $s$ of each collection; the thin dash-dotted lines mark the target power $0.8$ in panel (a) and the level $\alpha = 0.05$ in panel (b), the horizontal axis is logarithmic, and the markers are offset slightly so that their error bars stay apart}
\label{fig:realseed}
\end{figure}

\subsection{Two-method comparisons}
\label{sec:case2:pairs}

Table~\ref{tab:case2pairs} applies the full protocol to eight comparisons.
All eight are matched by seed, the zoo arms through the seeds $1$ to $10$ and the MultiBERTs arms through the pretraining seed, so the two arms are not independent as model~\eqref{eq:model} assumes; the table therefore reports the paired seed-level TOST, the one-sample TOST on the within-seed differences with $n - 1$ degrees of freedom and its $90\%$ interval, while the design count $n^{*}$ keeps the independent-arm design of a fresh experiment, the planning question of this paper.
That column is a floor: it assumes $\Delta = 0$, and at the observed differences the same formula asks for $116$ runs per method on SST-2 (learning rates) and $54$ on MNLI (steps), while for the RTE contrasts and the SST-2 step contrast, whose estimates lie at or beyond the margin, no run count certifies equivalence at $\delta = 0.01$.
The two zoo pairs, Adam against SGD and GELU against tanh at otherwise identical settings, are certified equivalent within $\delta = 0.01$ at $10$ matched pairs (paired TOST $p = 0.024$ and $0.021$; both $90\%$ intervals lie inside $\pm 0.009$).
These are the two one-factor contrasts of the hyperparameter zoo, optimiser and activation, with the remaining factors at the seed-zoo settings apart from the zoo-wide learning rate of $10^{-3}$; the two pairs share the Adam-GELU arm, so they are not independent comparisons, and both, together with the margins $0.01$ and $0.02$ carried over from Section~\ref{sec:case}, were fixed in advance of the analysis.

Run with the budgets of Section~\ref{sec:case}, the two zoo contrasts almost never certify equivalence: over all subsets of $n = 2$ matched pairs the paired TOST declares equivalence in $2$ to $4\%$ of them and for $n = 3$ in $3$ to $7\%$, in line with the retrospective power of the case study; at $n = 5$ the shares are $7\%$ and $9\%$.
The subsets are enumerated exactly (there are only $45$, $120$ and $252$ of them), so these shares describe the realised pool rather than fresh replications.

For the independent-arm planning question, a pilot of five runs per method followed by the design formula returns $n^{*} = 14$ and $10$ at the median over pilots, with $95\%$ pilot-to-pilot bands of $[5, 25]$ and $[5, 16]$, and $44$ to $47\%$ of the pilots return a value below the full-data one, so the point-estimate pilot rule of Section~\ref{sec:case} recovers the right order of magnitude relative to the full-pool plug-in reference, while a single pilot can still undersize the experiment.

Carried through end to end on the seed zoo, where $\Delta = 0$ by construction and $1000$ runs leave room for fresh confirmatory arms, the procedure falls short of its nominal power.
With a pilot of five runs per method, then the formula applied to the pooled pilot $s$ (prescriptions below five raised to five, which happens for $12\%$ of the pilots at the wider margin), then fresh runs of the prescribed size drawn without replacement from the remaining runs, equivalence is certified in a share $0.662 \pm 0.007$ of $5000$ replications at $\delta = 0.0173$ (for which the full-pool $s$ gives $n^{*} = 12$, and fresh arms of that fixed size certify in $0.794$) and in $0.598 \pm 0.007$ at $\delta = 0.01$ ($n^{*} = 33$, fixed-size rate $0.799$).
Two-thirds of the pilots undersize the experiment, with median prescriptions of $9$ and $24$ runs per method against $12$ and $33$; this share is not comparable with the $44$ to $47\%$ above, where the pilots were subsamples of the same ten runs that define the target.
Two mechanisms act: the pilot variance rests on eight degrees of freedom and lies below its expectation more often than above, which alone undersizes $0.54$ to $0.56$ of the pilots under a Gaussian law with the same $s$, and ten pilot runs contain none of the lowest $5\%$ of seeds $60\%$ of the time, so the heavy left tail of this seed law is usually missed as well.
Under the Gaussian law the shortfall persists at $0.718$ and $0.692$, so it is a property of sizing on a point estimate, sharpened by the tail.
Sizing on the upper $80\%$ confidence bound of the pooled pilot variance ($\chi^2$ on eight degrees of freedom) is the classical remedy for a variance taken from a pilot \citep{Browne1995,KieserWassmer1996}; here it is a normal-model inflation rule, since on this seed law the bound covers the full-pool variance in $61\%$ of the pilots rather than $80\%$.
It raises the rate to $0.802 \pm 0.006$ at $\delta = 0.0173$ and $0.779 \pm 0.006$ at $\delta = 0.01$, close to the fixed-size reference at the wider margin and short of it at the tighter one, with median prescriptions of $14$ and $39$ runs per method (median totals $38$ and $88$) and mean total costs, pilot included, of $50$ and $121$ runs against $34$ and $77$ for the point-estimate rule and $24$ and $66$ for a fixed-size design that knows the full-pool $s$.
Under the Gaussian law the same rule gives $0.894$ and $0.884$: the classical argument promises the intended power with probability at least $0.8$ over pilots rather than on average, and it presumes a valid variance bound and a calibrated power function, neither of which has been established for this seed law.

The MultiBERTs comparisons show the margin doing its work.
Between the learning rates $2 \cdot 10^{-5}$ and $5 \cdot 10^{-5}$, the two settings are certified equivalent on MNLI at both margins ($\hat\Delta = +0.0040$, interval $[+0.0028, +0.0052]$); on SST-2 the difference is distinguishable from zero and certified within $0.02$, while at $0.01$ the interval crosses the margin and equivalence is neither established nor excluded ($\hat\Delta = +0.0084$, interval $[+0.0067, +0.0101]$, a knife edge: the upper end exceeds the margin by $0.0001$); on RTE the two differ by ten accuracy points.
Between one and two million pretraining steps, at five runs per method, the second million buys about one point on SST-2 and MNLI ($\hat\Delta = -0.0101$ and $-0.0094$, intervals of width $0.0072$ and $0.0041$), certified within a margin of $0.02$, while at $0.01$ both intervals cross the margin, the SST-2 estimate just outside the margin and the MNLI estimate inside it, so equivalence at $0.01$ is neither established nor excluded; on RTE the five-pair interval is eleven points wide and settles nothing.

As a sensitivity check, the unpaired Welch-TOST returns $p = 0.010$ and $0.008$ for the zoo pairs (paired $0.024$ and $0.021$; same-seed correlations $-0.23$ and $-0.44$), $0.125$ for the learning-rate comparison on SST-2 (paired $0.059$; correlation $0.48$) and $0.53$, $0.23$ and $0.76$ for the step comparisons on SST-2, MNLI and RTE (paired $0.52$, $0.28$ and $0.85$; correlations $-0.43$, $-0.51$ and $+0.80$), with every verdict unchanged at both margins; ignoring a negative between-arm covariance understates the variance, and in the negatively correlated comparisons the unpaired intervals are indeed the narrower ones; on $10$ pairs the zoo correlations are not distinguishable from zero, so the pairing there is the analysis the design calls for rather than a finding.
A fresh experiment that trained matched pairs by design would size on the variance of the within-pair differences rather than on the independent-arm $n^{*}$ of the table; for the zoo contrasts this gives $18$ and $14$ pairs ($36$ and $28$ runs) against $14$ and $10$ runs per method ($28$ and $20$), so with the negative same-seed correlations observed here matching does not pay.

The eight comparisons at two margins are illustrative and carry no familywise control; as noted in Section~\ref{sec:discussion}, multiplicity inflates the rate of false certifications, so a study that needed separate verdicts on several contrasts would have to adjust for the multiplicity (one contrast tested at nested margins needs no adjustment, since the rejection regions are nested, and a single claim that all contrasts are equivalent is itself an intersection--union test).
Which of these verdicts is the useful one is a domain decision (Section~\ref{sec:discussion}); the protocol's contribution is to make the dependence on $\delta$ visible.

\begin{table}[t]
\caption{Eight two-method comparisons under the protocol, all matched by seed and analysed with the paired seed-level TOST on the within-seed differences: matched pairs $n$, estimated difference $\hat\Delta$ (first method minus second), paired $90\%$ interval, TOST verdict at $\delta = 0.01$ and $\delta = 0.02$, runs per method $n^{*}$ for $80\%$ power at $\delta = 0.01$ in a fresh experiment with independent arms, from the pooled $s$ of the two arms and computed under the counterfactual $\Delta = 0$ (the requirement grows without bound as the anticipated $|\Delta|$ approaches $\delta$, Section~\ref{sec:power}), and the share of the subsets of $2$, $3$ or $5$ matched pairs in which the paired TOST at $\delta = 0.01$ declares equivalence (exact enumeration over all $\binom{10}{n}$ or $\binom{25}{n}$ subsets; a subset whose differences coincide counts as a non-certification; not computed for the five-pair comparisons); the MultiBERTs rows use learning rates $2 \cdot 10^{-5}$ against $5 \cdot 10^{-5}$ at two million pretraining steps (rows \emph{lr}) and one against two million pretraining steps at $2 \cdot 10^{-5}$ (rows \emph{steps}); values of $n^{*}$ below $5$ are to be read as ``at least five'' (Remark~\ref{rem:powervalidity})}
\label{tab:case2pairs}
\centering\footnotesize\setlength{\tabcolsep}{3pt}
\begin{tabular}{llrrlccrrrr}
\hline
 & & & & & \multicolumn{2}{c}{TOST verdict} & & \multicolumn{3}{c}{certification rate} \\
Collection & Comparison & $n$ & $\hat\Delta$ & $90\%$ interval & $0.01$ & $0.02$ & $n^{*}$ & $n=2$ & $n=3$ & $n=5$ \\
\hline
zoo   & Adam vs SGD   & 10 & $-0.0002$ & $[-0.0080, +0.0077]$ & yes & yes & 14  & 0.04 & 0.03 & 0.07 \\
zoo   & GELU vs tanh  & 10 & $+0.0012$ & $[-0.0056, +0.0081]$ & yes & yes & 10  & 0.02 & 0.07 & 0.09 \\
RTE   & lr            & 25 & $+0.0983$ & $[+0.0737, +0.1229]$ & no  & no  & 397 & 0.00 & 0.00 & 0.00 \\
SST-2 & lr            & 25 & $+0.0084$ & $[+0.0067, +0.0101]$ & no  & yes & 6   & 0.04 & 0.04 & 0.08 \\
MNLI  & lr            & 25 & $+0.0040$ & $[+0.0028, +0.0052]$ & yes & yes & 3   & 0.28 & 0.67 & 0.99 \\
RTE   & steps         & 5  & $+0.0404$ & $[-0.0134, +0.0943]$ & no  & no  & 705 & \textendash & \textendash & \textendash \\
SST-2 & steps         & 5  & $-0.0101$ & $[-0.0137, -0.0065]$ & no  & yes & 3   & \textendash & \textendash & \textendash \\
MNLI  & steps         & 5  & $-0.0094$ & $[-0.0114, -0.0073]$ & no  & yes & 2   & \textendash & \textendash & \textendash \\
\hline
\end{tabular}
\end{table}

\subsection{Scope of this case study}
\label{sec:case2:scope}

The zoo networks are deliberately tiny (accuracy $0.49$ on CIFAR-10), a statistical testbed rather than models of practical interest, and the MultiBERTs configurations offer $25$ runs, so that in the single-configuration check of Section~\ref{sec:case2:formula} two disjoint arms of $n = 12$ use $24$ of them and successive draws differ only in how nearly the same runs are split.
More generally the subsampled rates condition on the observed run pool: the quoted Monte-Carlo standard errors measure the resampling error of that finite-pool rate, while the uncertainty about the seed law itself, which is estimated from only $25$ runs, is not captured by them and may dominate.
We therefore read the zoo for the shape of the seed law and the GLUE tasks for the regimes and the comparisons.
A nested decomposition of the $125$ MultiBERTs runs in which both seeds vary attributes more variance to the fine-tuning seed than to the pretraining seed on all three tasks (intraclass correlation of the pretraining seed $0.38$ on SST-2 and MNLI, and $0.13$ on RTE, where it is only marginally distinguishable from zero: $F = 1.73$ on $24$ and $100$ degrees of freedom, $p = 0.03$); the per-run variance at the fixed fine-tuning seed nevertheless lies within $20\%$ of the total variance of that collection on all three tasks, a check of the variance only, not of the method differences or of the tails.
This is why the run unit must include every stochastic stage that the compared methods re-run; a comparison that reuses one pretrained checkpoint estimates a different and narrower quantity.
Proposition~\ref{prop:shared} also asks the test collection to be independent of the runs; GLUE development sets are commonly used for checkpoint selection during fine-tuning, and the MultiBERTs release does not document whether this occurred, so that premise is plausible rather than guaranteed there.
Accuracies are discrete, so at $n = 2$ the two runs of an arm occasionally have identical accuracies in both arms, or the two paired differences coincide; the standard error then vanishes and the TOST statistic is undefined, in up to one draw in a hundred in the single-configuration check of Section~\ref{sec:case2:formula} and in up to six in a hundred of the subsets behind Table~\ref{tab:case2pairs} ($n = 2$ on SST-2, where $18$ of the $300$ pairs of within-seed differences coincide).
Such draws are counted as non-certifications.

\section{Discussion and limitations}
\label{sec:discussion}

\emph{Approximations in the power theory.}
Proposition~\ref{prop:power} treats $\hat\Delta$ as normal and $\widehat{\mathrm{SE}}$ as known.
The joint effect of both approximations is mapped rather than assumed: the formula is accurate in the design-relevant range and errs predominantly on the conservative side outside it, the exception being moderate-power cells at $n = 3$, while the realised size at the boundaries stays at or near $\alpha$ with no evidence of systematic inflation, under Gaussian and non-Gaussian seed effects alike (Remark~\ref{rem:powervalidity}, Fig.~\ref{fig:power}).
The real seed laws of Section~\ref{sec:case2} confirm the pattern with one qualification: under the heavy-tailed law of the zoo the error of the formula within the design-relevant range reached $0.05$ in the conservative direction and $0.03$ in the other, so the safety margin it offers is thinner than the Gaussian grid suggests, while the realised size stayed at or near $\alpha$, its largest boundary value of $0.055$ lying within Monte-Carlo resolution of the level.
For design use we regard the conservative direction as a feature, with one rule attached: recommendations that land below five runs are raised to five (Remark~\ref{rem:powervalidity}).

\emph{Bootstrap validity is regime-dependent.}
We claim no general consistency for the two-level scheme; its coverage is near-nominal only when the variance levels are balanced (Remark~\ref{rem:bootvalidity}, Fig.~\ref{fig:coverage}).
The map itself is finite: one binormal generator family, one test-sample size, $B = 400$ and equal method dispersions; the test-set-size and variance-ratio axes remain unexplored.
The anti-conservative regime, very few runs with dominant seed noise, is precisely where small-$n$ practice lives, so at $n \le 3$ we recommend the seed-level Welch interval and reserve the bootstrap for composite statistics at moderate $n$.
The Welch interval is approximate, since the unequal-variance problem admits no exact interval of the standard $t$ form, but it holds its level across the equal-variance designs simulated here.

\emph{Dependence between test sets.}
The per-set AUROCs of one run share the ID sample, so they are correlated; the case study accounts for this in $\bar v$ through the cross-set covariance of placement values.
Per-set TOSTs, however, remain dependent across OOD sets, and we make no multiplicity claim for testing several margins at once; the macro-average, treated as a single statistic, sidesteps the multiplicity by changing the estimand.
It does not certify the sets separately: in Table~\ref{tab:runs} the mean per-set differences reach $+0.028$ (SVHN) and $-0.012$ (LSUN-R), opposite in sign and both beyond the margin, so per-set equivalence claims would need margins and a multiplicity-aware procedure of their own.

\emph{Uncertainty of the plug-ins.}
At $n = 2$ the pooled $s^2$ carries roughly one effective degree of freedom, and the design answer inherits that uncertainty (Section~\ref{sec:case}): $n^{*}$ is then an order-of-magnitude statement.
The anticipated difference is a second, equally consequential plug-in: sizing at $\Delta = 0$ gives the smallest possible $n$, and the requirement grows quickly as the anticipated $|\Delta|$ approaches the margin (Section~\ref{sec:case}).
The two-stage remedy, a pilot of $4$--$6$ runs followed by the formula, is cheap relative to the certification experiment itself.
Carried through on the real seed law of Section~\ref{sec:case2} with the pilot $s$ plugged in as a point estimate, however, it delivers $0.60$ to $0.66$ against a nominal $0.80$, since most pilots undersize; sizing on the upper $80\%$ confidence bound of the pilot variance \citep{Browne1995,KieserWassmer1996} raises it to $0.78$ to $0.80$ at about one and a half times the run count of the point-estimate rule, roughly twice that of a fixed-size design with the full-pool variance (Section~\ref{sec:case2:pairs}); this rests on one collection, two margins and $\Delta = 0$, a planning sensitivity analysis rather than a general guarantee.
The methods may also differ in dispersion, as they appear to in the case study; Welch statistics absorb this in testing, and the unequal-variance design formula of Section~\ref{sec:power} prices it at one extra run on the case-study estimates (Section~\ref{sec:case}).

\emph{Beyond AUROC.}
The seed-level layer, that is, model~\eqref{eq:model}, the TOST and the power function, carries over to other per-run scalar metrics: accuracy, macro-$F_1$, calibration error.
Section~\ref{sec:case2} exercises it with accuracy, whose item-level variance is binomial, and makes the shared-test-set regime visible in which that binomial variance exceeds the observed per-run variance and must not be added to it.
The variance decomposition asks only for finite variances, but the Welch-TOST and the shifted-$t$ heuristic additionally lean on approximate normality of the run means, so a metric with a strongly skewed or heavy-tailed per-run law at small $n$ calls for a simulation check of its own before the design formula is trusted.
Finite variances alone do not protect against a multimodal run law either: if a small fraction of runs fails outright, so that the run law is a mixture with a rare large shift, five runs per method will often contain no failure at all, and equivalence can then be certified far more often than $\alpha$ allows even when the true difference lies at or beyond the margin.
The per-run values should be inspected for such failures, whose frequency is an estimand of its own, before the protocol is applied to their mean, bearing in mind that a short failure-free pilot cannot exclude a rare component.
Only the inner-level estimator is AUROC-specific; for metrics that average per-item contributions, $\bar v$ can be estimated from the empirical influence values, while metrics without such structure require a metric-specific inner estimator, with the outer layer unchanged.

\emph{Choice of the margin.}
The margin $\delta$ must be fixed before the experiment \citep{Lakens2017}; we treat it as a domain decision rather than a statistical one, the smallest difference that would change the practical decision, and not as the resource-based fallback of the smallest detectable difference that \citet{Lakens2017} allows when no substantive bound is available.
Reporting the $(1 - 2\alpha)$ interval next to $\delta$ (Proposition~\ref{prop:duality}) keeps the conclusion transparent under any reader's margin.

\section{Conclusions}
\label{sec:conclusions}

Certifying that two training methods are practically equivalent is a design problem before it is a testing problem.
The protocol proposed here (a margin fixed in advance, per-run AUROCs as the unit of analysis, a Welch-TOST with its dual confidence interval, and a two-level bootstrap for composite statistics) makes the certification claim precise, and the power theory prices it in training runs: $n \ge 2 (\sigma^2 + \bar v)(z_{1-\alpha} + z_{(1+\pi^{*})/2})^2 / \delta^2$.
Once the test set is large enough for the image-level term to be subordinate to the seed-level term, the regime of the case study, the price is dominated by the seed-level dispersion, which no test-set enlargement can reduce.
In the case study this reads: the published TOST of \citet{Knopov2026jem} at $n = 2$, which failed to declare equivalence, had, under the variance calibrated to its own runs, a realised certification probability of roughly three per~cent, and a certification of $\pm 0.01$ would take about $12$ runs per method even under exact equality; that statement could have been made after a small pilot, before the confirmatory experiment.
On two public collections with up to a thousand independently trained runs the same formula tracks the realised certification probability to within $0.05$ wherever its design-relevant range applies, under a heavy-tailed real seed law, keeps the realised size at or near $\alpha$, and resolves comparisons that two or three runs leave undecided.
The two-stage experiment also shows that a pilot variance plugged in as a point estimate undersizes the experiment, and that a normal-model upper confidence bound on it raised the achieved rates from $0.60$ and $0.66$ to $0.78$ and $0.80$ at about half again the run count of the point-estimate rule.
Natural continuations are multiplicity-aware equivalence over several test sets, sequential designs that spend runs adaptively, and inner-level estimators for metrics beyond AUROC.

\backmatter

\newpage
\section*{Declarations}
\vspace{6pt plus 1ex minus .2ex}

\bmhead{Funding}
The authors received no specific funding for this work.

\bmhead{Competing interests}
The authors declare no competing interests.

\bmhead{Ethics approval}
Not applicable.
This study involves no human participants and no animals; it uses simulated data and publicly available benchmark datasets only.

\bmhead{Consent to participate}
Not applicable.
No human participants were involved in this study.

\bmhead{Consent to publish}
Not applicable.
The manuscript contains no data, images or videos relating to an identifiable individual person.

\bmhead{Data availability}
The per-item out-of-distribution scores of the four case-study runs are archived on Zenodo with the source study \citep{Knopov2026jem}, \url{https://doi.org/10.5281/zenodo.21872469}.
Every simulation grid reported here is regenerated from the root seeds recorded in the result files shipped with the code repository below.
The Model Zoos checkpoints of Section~\ref{sec:case2} are available from Zenodo under the CC BY 4.0 licence \citep{ModelZoosData2022}, and the MultiBERTs per-example predictions from the authors' public release \citep{Sellam2022}, \url{https://storage.googleapis.com/multiberts/public/example-predictions/}, whose data files carry no stated licence (the accompanying repository is under Apache-2.0); the per-run accuracies derived from both are included with the code.
The third-party benchmark datasets involved (CIFAR-10, SVHN, CIFAR-100, DTD, LSUN-R, iSUN, GLUE) are publicly available.

\bmhead{Materials availability}
Not applicable.

\bmhead{Code availability}
The \texttt{equivtest} package, the simulation scripts, the scripts of both case studies and the figure builder are available at \url{https://github.com/DmytroKnopov/equivtest} (release v0.2.0, archived at \url{https://doi.org/10.5281/zenodo.23421510}); the test suite includes regression tests against the published values of \citet{Knopov2026jem}.
The package requires Python~$\geq$~3.10 with NumPy and SciPy; reproducing the checkpoint-based zoo analysis of Section~\ref{sec:case2} additionally requires PyTorch and \texttt{torchvision}, and the download steps of both collections use \texttt{requests}.

\bmhead{Author contributions}
Contributions follow the CRediT taxonomy.
Dmytro Knopov: conceptualization, methodology, software, formal analysis, investigation, validation, data curation, visualization, writing (original draft).
Nadiya Shvai: conceptualization, supervision, writing (review and editing).
Both authors read and approved the final manuscript.

\bmhead{Use of AI-assisted tools}
During the preparation of this manuscript, the first author used large language models from Anthropic (Claude) and OpenAI (GPT) to assist with drafting, editing and reviewing the text, organising the related literature and developing and debugging the code for the simulations and the case studies.
The research ideas, methodology, analyses, results and conclusions reported here are the authors' own, and the authors take full responsibility for them.

\bibliography{refs}

\end{document}